%% file: example_paper.tex
\documentclass{article}

\usepackage{microtype}
\usepackage{graphicx}
\usepackage{subfigure}
\usepackage{booktabs} 

\usepackage{amsfonts}
\usepackage{amsthm}
\usepackage{amsmath}

\usepackage{hyperref}

\newtheorem{theorem}{Theorem}[section]
\newtheorem{corollary}{Corollary}[section]
\newtheorem{lemma}[theorem]{Lemma}
\newtheorem{definition}{Definition}[section]
\newtheorem{example}{Example}[section]
\newtheorem{remark}{Remark}[section]

\usepackage[accepted]{icml2019}

\icmltitlerunning{Online Convex Optimization in Adversarial MDPs}

\begin{document}

\twocolumn[
\icmltitle{Online Convex Optimization in Adversarial Markov Decision Processes}



\icmlsetsymbol{equal}{*}

\begin{icmlauthorlist}
\icmlauthor{Aviv Rosenberg}{tau}
\icmlauthor{Yishay Mansour}{tau,goo}
\end{icmlauthorlist}

\icmlaffiliation{tau}{Tel Aviv University, Israel}
\icmlaffiliation{goo}{Google Research, Tel Aviv, Israel}

\icmlcorrespondingauthor{Aviv Rosenberg}{avivros007@gmail.com}
\icmlcorrespondingauthor{Yishay Mansour}{mansour.yishay@gmail.com}

\icmlkeywords{Machine Learning, ICML}

\vskip 0.3in
]



\printAffiliationsAndNotice{} 

\begin{abstract}
We consider online learning in episodic loop-free Markov decision processes (MDPs), where the  loss function can change arbitrarily between episodes, and the transition function is not known to the learner.
We show $\tilde{O}(L|X|\sqrt{|A|T})$ regret bound, where $T$ is the number of episodes, $X$ is the state space, $A$ is the action space, and $L$ is the length of each episode. 
Our online algorithm is implemented using entropic regularization methodology, which allows to extend the original adversarial MDP model to handle convex performance criteria 
(different ways to aggregate the losses of a single episode)
, as well as improve previous regret bounds.
\end{abstract}

\section{Introduction}

Markov Decision processes \cite{mdp} have been widely used to model reinforcement learning problems - problems involving sequential decision making in a stochastic environment. In this model both the losses and dynamics of the environment are assumed to be stationary over time.
However, in real world applications, the losses might change over time, even throughout the learning process.

The adversarial MDP model \cite{evendar} was proposed to address these issues. In this model, the loss function can change arbitrarily (while still assuming a fixed stochastic transition function). The learner's objective is to
minimize its average loss during the learning process, and its performance is measured by the regret - comparing to the best stationary policy in hindsight. These ideas originate from online learning problems \cite{bianchi-lugosi} -  where, in each round, the learner selects an action before knowing the current loss function.

BGP routing is considered as a motivating example in the full version of the paper.

We propose a novel algorithm for the adversarial MDP model where the transition function is unknown to the learner and the losses change arbitrarily over time.
Our algorithm, UC-O-REPS, uses two important ingredients, the first is 
Online Mirror Descent (OMD) \cite{omd} and the second is  UCRL-2 \cite{jaksch}.
A major challenge in this work is to handle convex performance criteria, which model different ways of aggregating the losses of each episode.
In order to handle convex performance criteria, we use the methodology of OMD, which is widely used for online convex optimization, and we implement it in the adversarial MDP setting.
In order to overcome the unknown dynamics (stochastic transition function) we incorporate techniques from UCRL-2.

Our main contribution is extending the adversarial MDP model to include convex performance criteria,
and showing that our algorithm, UC-O-REPS, achieves near-optimal regret bounds in the general model.
This is an important extension since different applications have different optimization criteria, other than minimizing the expected average loss. Examples include risk-sensitive objectives and robust objectives (that combine multiple loss functions).
In addition,
we improve the known regret bound of \citet{ofpl} for the expected average loss from $\tilde{O}(L|X||A|\sqrt{T})$ to achieve $\tilde{O}(L|X|\sqrt{|A|T})$, which is especially important for large action spaces. Our bounds also hold with high probability, and not only in expectation.
Our algorithm builds on a simple entropic regularization method, and the main challenge is the analysis of the regret and computational complexity.

\subsection{Related Work}

The works of \citet{jaksch} and \citet{regal} assume an unknown fixed MDP, and achieve a $\tilde{O}(L|X|\sqrt{|A|T})$ regret compared to the optimal policy. A recent work by \citet{azar} achieves $\tilde{O}(\sqrt{L|X||A|T})$ regret for large enough $T$, which is optimal \cite{jaksch}. We remark that the lower bound of $\Omega(\sqrt{L|X||A|T})$ by \citet{jaksch} shows that our regret bound is optimal with respect to the number of time steps $T$ and actions $|A|$.

The work of \citet{evendar}, which presented the adversarial MDP model, assumes full knowledge of the transition function and full information feedback about the losses. They propose an algorithm, MDP-E, which uses an experts algorithm in each state and achieves $O(\tau^2 \sqrt{T\ln{|A|}})$ regret, where $\tau$ is a bound on the mixing time of the MDP. Another early work in this setting, by \citet{yu}, achieves an $O(T^{2/3})$ regret.

In the bandit setting, the learner observes only the losses related to its actions, i.e., a bandit feedback.
The work of \citet{ossp} achieves an $O(L^2\sqrt{T|A|}/\alpha)$ regret, where $\alpha > 0$ is a lower bound on the steady state probability to reach some state $x$ under some policy $\pi$.
Later \citet{nipsneu} eliminate the dependence on $\alpha$ but achieve only 
$\tilde{O}(T^{2/3})$ regret.
A later work, by \citet{zimin}, proposed the O-REPS algorithm
which guarantees an $\tilde{O}(\sqrt{L|X||A|T})$ regret.

The only work that considers the setting of unknown transition function in an adversarial MDP is  \citet{ofpl}. They propose an algorithm, Follow the Perturbed Optimistic Policy (FPOP), which builds on Follow the Perturbed Leader \cite{kalai},
and achieves $\tilde{O}(L|X||A|\sqrt{T})$ regret.

The rest of the paper is organized as follows. Section \ref{sec:prob} presents the formal model and problem. Section \ref{sec:oc-m} presents the concept of occupancy measures, which will enable us to reformulate the problem as an instance of online convex optimization. Section \ref{sec:alg} describes our algorithm and its efficient implementation. Section \ref{sec:reg} proves our algorithm's regret bound.

\section{Problem Formulation}
\label{sec:prob}

An episodic loop-free adversarial MDP is defined by a tuple $M= \left( X,A,P, \{ \ell_t \}_{t=1}^T \right) $, where $X$ and $A$ are the finite state and action spaces, and $P:X \times A \times X \rightarrow[0,1]$ is the transition function such that $P(x'|x,a)$ is the probability to move to state $x'$ when performing action $a$ in state $x$.

We assume that the state space can be decomposed into $L$ non-intersecting layers $X_0 , \dots , X_L$ such that the first and the last layers are singletons, i.e., $X_0=\{x_0\}$ and $X_L = \{x_L\}$. Furthermore, the loop-free assumption means that transitions are only possible between consecutive layers.
These assumptions are not necessary, but they simplify some arguments and have a nice interpretation as a game with $L$ steps played for $T$ times.

Let $\{ \ell_t \}_{t=1}^T$ be a sequence of loss functions describing the losses at each episode, i.e., $\ell_t: X \times A \times X \rightarrow[0,1]^d$. We do not make any statistical assumption on the loss functions, i.e., they can be chosen arbitrarily. Notice that the losses might be multidimensional which can be useful for modeling multiple losses at the same time. Moreover, the learner does not suffer the losses directly, instead they are aggregated using some performance criterion (defined later).

The interaction between the learner and the environment is described in Algorithm \ref{alg:env-interaction}. It proceeds in episodes, where in each episode the learner starts in state $x_0$ and moves forward across the consecutive layers until it reaches state $x_L$. The learner's task is to select an action at each state it visits. Alternatively, we can say that its task at each episode is to choose a stationary (stochastic) policy ,which is a mapping $\pi : X\times A \rightarrow[0,1]$, where $\pi(a|x)$ gives the probability that action $a$ is selected in state $x$.

We denote by $U$ a trajectory through the consecutive layers from $x_0$ to $x_L$, and by $\ell(U)$ the sequence of losses obtained in this trajectory (with respect to loss function $\ell$), i.e.,
\begin{align*}
    U & = \left( x_0, a_0,x_1,a_1,\dots,x_{L-1},a_{L-1}, x_L \right)
    \\
    \ell(U) & = \Bigl\{
    \ell(x_k,a_k,x_{k+1})
    \Bigl\}_{k=0}^{L-1}
\end{align*}
Moreover, we use the notation $\mathbb{E} \left[ \ell(U) | P,\pi \right]$ for the expectation of the losses obtained over trajectories that are generated using transition function $P$ and policy $\pi$. That is,
action $a_k$ is chosen using $\pi(\cdot|x_k)$ and state $x_{k+1}$ is drawn from distribution $P(\cdot|x_k,a_k)$.

The goal of the learner is to minimize its total loss with respect to some performance criterion $\mathcal{C}$, i.e.,
$$\hat{L}_{1:T}^{\mathcal{C}} ( \{ \ell_t \}_{t=1}^T) = \sum_{t=1}^T \mathcal{C} \left(
\mathbb{E}\left[ \ell_t(U) | P,\pi_t \right] \right)$$
where $\pi_t$ is the policy chosen by the learner in episode $t$, and $\mathcal{C}: ( \mathbb{R}^d )^L \rightarrow \mathbb{R}_{\geq0}$ is the performance criterion, that aggregates the losses of each episode. 

\begin{algorithm}
   \caption{Learner-Environment Interaction}
   \label{alg:env-interaction}
\begin{algorithmic}
   \STATE {\bfseries Parameters:} MDP $M= \left( X,A,P, \{ \ell_t \}_{t=1}^T \right)$ and performance criterion $\mathcal{C}$
   \FOR{$t=1$ {\bfseries to} $T$}
   \STATE learner starts in state $x_0^{(t)}=x_0$
   \FOR{$k=0$ {\bfseries to} $L-1$}
   \STATE learner chooses action $a_k^{(t)} \in A$
   \STATE environment draws new state $x_{k+1}^{(t)} \sim P(\cdot | x_k^{(t)},a_k^{(t)}) $
   \STATE learner observes state $x_{k+1}^{(t)}$
   \ENDFOR
   \STATE loss function $\ell_t$ is exposed to learner
   \ENDFOR
\end{algorithmic}
\end{algorithm}

Here are a few interesting and important examples for performance criteria, that our algorithm is able to handle.

\begin{example}
The simplest and most useful example is the total expected loss (TEL) performance criterion, which (to the best of our knowledge) has been the only performance criterion studied so far. 
Losses are $1$-dimension, i.e.,  $d=1$, and the criterion is defined as follows,
$$
\mathcal{C}^{TEL} \left( \{ v_k \}_{k=0}^{L-1} \right) =  \sum_{k=0}^{L-1} v_k \qquad (v_k \in \mathbb{R})
$$
\end{example}

\begin{example}
We can use the performance criterion to minimize the worst case loss when there are multiple loss functions. Here each dimension of the losses is considered as an individual loss function, and the learner's objective is a min-max criterion, i.e.,
$$
\mathcal{C}^{MM} \left( \{ v_k \}_{k=0}^{L-1} \right) = \max_{1\leq i \leq d} \sum_{k=0}^{L-1} v_k{[i]} \qquad (v_k \in \mathbb{R}^d
)$$
\end{example}

\begin{example}
We can use the performance criterion for a notion of risk-sensitivity. Here losses are $1$-dimension and we want to minimize a trade-off between the loss and the risk. Specifically, given a trade-off parameter $0 \leq  \alpha \leq 1$ and a risk parameter $c>1$, the performance criterion is
$$
\mathcal{C}_{\alpha,c}^{RISK} \left( \{ v_k \}_{k=0}^{L-1} \right) = \alpha \left(  \sum_{k=0}^{L-1} v_k \right)^c
+ (1- \alpha) \sum_{k=0}^{L-1} \left( v_k \right)^c
$$
\end{example}

The performance of the learner will be measured by comparison to the best stationary policy with respect to the chosen performance criterion. For a policy $\pi$ we define its total loss with respect to some performance criterion $\mathcal{C}$ as
$$
L_{1:T}^\mathcal{C}(\pi;\{\ell_t\}_{t=1}^T)=\sum_{t=1}^{T} \mathcal{C} \left( \mathbb{E}\left[ \ell_t(U) | P,\pi \right] \right)
$$
Thus the learner's regret is defined as follows,
$$
\hat{R}_{1:T}^\mathcal{C}=\hat{L}_{1:T}^\mathcal{C}(\{\ell_t\}_{t=1}^T)-\min_{\pi}L_{1:T}^\mathcal{C}(\pi;\{\ell_t\}_{t=1}^T)
$$
where the minimum is taken over all stationary stochastic policies.

\begin{remark}
Note that if the dynamics were known to the learner, it would not need to observe the trajectory $U_t$ at each episode $t$, since it could compute its performance criterion using $\ell_t$, $\pi_t$ and $P$.
In this case, we actually reduce the problem to online learning in the space of the policies.
When the dynamics are unknown, the learner uses the observed trajectories $U_t$ to estimate the transition function $P$, which enables it to estimate its performance criterion.
\end{remark}

\section{Occupancy Measures}
\label{sec:oc-m}

We would like to reformulate the learner's objective in order to approach the problem with techniques from online learning. For this purpose we introduce the concept of occupancy measures \cite{zimin} on the space $X\times A\times X$. For a policy $\pi$ and a transition function $P$ we define the occupancy measure $q^{P,\pi}$ as follows:
$$
q^{P,\pi}(x,a,x')=\Pr\left[x_k=x,a_k=a,x_{k+1}=x' | P,\pi \right]
$$
where $x \in X_k$ and $x' \in X_{k+1}$. Another notation we will be using is $k(x)$ for the index of the layer that $x$ belongs to.

We start with two basic properties that hold for every occupancy measure $q$. From the loop-free assumption we know that in each episode the learner will go through every layer. Therefore, for every $k=0,\dots,L-1$,
\begin{equation}
    \label{eq:sum_one}
    \sum_{x\in X_{k}}\sum_{a\in A}\sum_{x'\in X_{k+1}}q(x,a,x') = 1
\end{equation}

Moreover, the probability to enter a state when coming from the previous layer is exactly the probability to visit that state. Thus, for every $k=1,\dots,L-1$ and every $x \in X_k$,
\begin{equation}
\label{eq:flow}
\sum_{x'\in X_{k+1}}\sum_{a\in A}q(x,a,x') = \sum_{x'\in X_{k-1}}\sum_{a\in A}q(x',a,x)   
\end{equation}

Notice that every occupancy measure $q$ induces a transition function and a policy. We denote them as $P^q$ and $\pi^q$ respectively, and they can be computed as follows:
\begin{align*}
    P^q(x'|x,a) & = \frac{q(x,a,x')}{\sum_{y\in X_{k(x)+1}}q(x,a,y)}
    \\
    \pi^q(a|x) & = \frac{\sum_{x'\in X_{k(x)+1}}q(x,a,x')}{\sum_{b\in A}\sum_{x'\in X_{k(x)+1}}q(x,b,x')}
\end{align*}

We denote the set of all occupancy measures of an MDP $M$ as $\Delta(M)$. The following lemma characterizes $\Delta(M)$ and its proof is straightforward. 
\begin{lemma}
For every $q\in[0,1]^{|X|\times|A|\times|X|}$ it holds that $q \in \Delta(M)$ if and only if \eqref{eq:sum_one} and \eqref{eq:flow} hold, and $P^q=P$ (where $P$ is the transition function of $M$).
\end{lemma}

We can use occupancy measures to reformulate the regret. We say that a performance criterion $\mathcal{C}$ is convexly-measurable if there exists some convex function $f^\mathcal{C}:[0,1]^{|X|\times|A|\times|X|} \rightarrow \mathbb{R}_{\geq0}$, such that
$$
\mathcal{C} \left( \mathbb{E}\left[ \ell(U) | P,\pi \right] \right) = f^\mathcal{C}(q^{P,\pi};\ell)
$$
holds for every policy $\pi$ and every transition function $P$. We call $f^\mathcal{C}$ the criterion function of $\mathcal{C}$. 
Since our algorithm requires only the criterion function,  performance criteria can also be defined implicitly through criterion functions.

If we redefine the task of the learner from having to select individual actions (or policies) to having to select occupancy measures $q_t \in \Delta(M)$ in each episode $t$, for convexly-measurable performance criteria we can rewrite the regret to obtain an instance of online convex optimization with decision space $\Delta(M)$, i.e.,
\begin{align*}
    \hat{R}_{1:T}^\mathcal{C} & = \hat{L}_{1:T}^\mathcal{C}(\{\ell_t\}_{t=1}^T)-\min_{\pi}L_{1:T}^\mathcal{C}(\pi;\{\ell_t\}_{t=1}^T)
    \\
    & = \sum_{t=1}^T f^\mathcal{C}(q_t;\ell_t) - \min_{q \in \Delta(M)} \sum_{t=1}^T f^\mathcal{C}(q;\ell_t)
    \\
    & = \max_{q \in \Delta(M)} \sum_{t=1}^T f^\mathcal{C}(q_t;\ell_t)-f^\mathcal{C}(q;\ell_t)
\end{align*}

The following lemma shows that all performance criterion examples presented in the previous section are indeed convexly-measurable, and gives a way to build more convexly-measurable performance criteria. 

\begin{lemma}
If a performance criterion $\mathcal{C}$ has the following form,
$$
\mathcal{C} \left( \{ v_k \}_{k=0}^{L-1} \right) = g \left( \Bigl\{ \sum_{k=0}^{L-1} h_j(v_k) \Bigl\}_{j=1}^m \right)
$$
where $v_k\in \mathbb{R}^d $,  $h_j:\mathbb{R}^d \rightarrow \mathbb{R}_{\geq0}$ are arbitrary functions and $g: \mathbb{R}^m \rightarrow \mathbb{R}_{\geq0}$ is a convex function, then $\mathcal{C}$ can be modeled as a convexly-measurable performance criterion.
\end{lemma}

\begin{proof}
For any loss function $\ell'$, policy $\pi$ and transition function $P$, we have that
\begin{align*}
    \mathcal{C}^{TEL} ( & \mathbb{E} [ \ell'(U)  | P,\pi ] ) = \sum_{k=0}^{L-1} 
    \mathbb{E}\left[ \ell'(x_{k},a_{k},x_{k+1}) \Bigl| P,\pi
    \right]
    \\
    & = \mathbb{E}\left[ \sum_{k=0}^{L-1} \ell'(x_{k},a_{k},x_{k+1}) \Bigl| P,\pi\right]
    \\
    &= \sum_{x,a,x'} q^{P,\pi}(x,a,x')\ell'(x,a,x')
     \stackrel{def}{=} \langle q^{P,\pi},\ell' \rangle
\end{align*}
Therefore the criterion function of $\mathcal{C}^{TEL}$ is $f^{\mathcal{C}^{TEL}}(q;\ell) = \left< q,\ell \right>$.
We can model $\mathcal{C}$ with $m$-dimension losses, such that dimension $j$ features loss function $h_j(\ell)$, and then $\mathcal{C}$ just needs to sum up the $L$ losses and apply $g$. Thus, the criterion function of $\mathcal{C}$ will be
$$
f^\mathcal{C}(q;\ell) = g \left( \left\{ 
\langle q,h_j(\ell) \rangle
\right\}_{j=1}^m \right)
$$
Finally, $f^\mathcal{C}$ is convex because the composition of a convex function and a linear function is convex \cite{boyd}.
\end{proof}

\section{The Algorithm}
\label{sec:alg}

We call our algorithm, which is presented in algorithms \ref{alg:uc-o-reps} and \ref{alg:comp-pol-prod}, ``Upper Confidence Online Relative Entropy Policy Search'' (UC-O-REPS). It is inspired by the O-REPS algorithm \cite{zimin} in the sense that it picks occupancy measures instead of policies. However, unlike our algorithm, O-REPS assumes full knowledge of the transition function. To the best of our knowledge, the only algorithm that handles unknown transition probabilities in adversarial MDPs is FPOP \cite{ofpl}, which uses a Follow the Pertubed Leader method \cite{kalai} in the space of the policies.

Recall that the adversarial  MDP has a stochastic element - the transition function, and an adversarial element - the loss functions.

To handle the stochastic transition function we use the framework of epochs and confidence sets, first introduced by the UCRL-2 algorithm \cite{jaksch}.
In this framework, the algorithm maintains confidence sets that contain the actual MDP with high probability, but also shrink as time progresses.We translated this method to the occupancy measures space, and the full details can be found in Section \ref{sec:con-set}.

The core of the algorithm is the way we choose the occupancy measure for each episode from within the confidence set. This is done by the Online Mirror Descent method \cite{omd} for online linear optimization, since we deal with an arbitrary sequence of loss functions. The full details of adapting OMD to our setting can be found in Section \ref{sec:opt}.

The combination of these two methods is done using an important principle in reinforcement learning - ``optimism in face of uncertainty''. On the one hand, we keep confidence sets to handle the uncertainty, but on the other hand, within these confidence sets, we solve an OMD optimization problem optimistically (without thinking about the transition function estimation).

\subsection{Confidence Sets}
\label{sec:con-set}

Since the learner does not know the transition function, it has to estimate $P$ from its experience. Using this estimate we define confidence sets, and choose occupancy measures from within them. Notice that these occupancy measures might not be in $\Delta(M)$, i.e., their induced transition function may differ from $P$. Nevertheless, we can still use them to compute policies and execute those policies.

The algorithm proceeds in epochs of random length, and in the beginning of each epoch the confidence set is updated. The first epoch $E_{1}$ starts at episode $t=1$, and each epoch $E_{i}$ ends when the number of visits at some state-action pair $(x,a)$ is doubled. Let $t_{i}$ denote the index of the first episode in epoch $E_{i}$, and $i(t)$ denote the index of the epoch that includes episode $t$. Let $N_{i}(x,a)$ and $M_{i}(x'|x,a)$ denote the number of times state-action pair $(x,a)$ was visited and the number of times this event was followed by a transition to $x'$ up to episode $t_{i}$, respectively. That is
\begin{align*}
    N_{i}(x,a)&=\sum_{s=1}^{t_{i}-1}\mathbb{I}\left\{ x_{k}^{(s)}=x,a_{k}^{(s)}=a\right\} 
    \\
    M_{i}(x'|x,a)&=\sum_{s=1}^{t_{i}-1}\mathbb{I}\left\{ x_{k}^{(s)}=x,a_{k}^{(s)}=a,x_{k+1}^{(s)}=x'\right\} \nonumber
\end{align*}
where $k=k(x)$.

Our estimate $\bar{P}_i$ for the transition function in epoch $E_{i}$ is
$$
\bar{P}_{i}(x'|x,a)=\frac{M_{i}(x'|x,a)}{\max\left\{ 1,N_{i}(x,a)\right\} }
$$
and we define our confidence set $\Delta(M,i)$ in epoch $E_i$ to include all the occupancy measures that their induced transition function is ``close enough'' to $\bar{P}_i$. More formally, given a confidence parameter $\delta > 0$, we define
$$
\epsilon_{i}(x,a)=\sqrt{\frac{2 |X_{k(x)+1}| \ln\frac{T|X||A|}{\delta}}{\max\{1,N_{i}(x,a)\}}}
$$
and say that $\Delta(M,i)$ consists of all $q\in[0,1]^{|X|\times|A|\times|X|}$ for which \eqref{eq:sum_one} and \eqref{eq:flow} hold, and
\begin{equation}
    \label{eq:l1_dist}
    \left\lVert P^q(\cdot|x,a) - \bar{P}_i(\cdot|x,a) \right\rVert_1 \leq \epsilon_{i}(x,a)
\end{equation}
for every $(x,a) \in X \times A$.

Notice that these confidence sets shrink as time progresses, but the following lemma \cite{jaksch,ofpl} shows that they still contain $\Delta(M)$ with high probability.
\begin{lemma}
\label{lem:high_prob}
For any $0<\delta<1$
$$\left\lVert P(\cdot |x,a) - \bar{P}_i(\cdot|x,a) \right\rVert_1 \leq \sqrt{\frac{2 |X_{k(x)+1}| \ln\frac{T|X||A|}{\delta}}{\max\{1,N_{i}(x,a)\}}}$$
holds with probability at least $1-\delta$ simultaneously for all $(x,a) \in X \times A$ and all epochs.
\end{lemma}

\subsection{Optimization Problem}
\label{sec:opt}

In order to choose the occupancy measure $q_t$ for episode $t$, the algorithm follows the OMD method. The idea behind this method is to choose an occupancy measure that minimizes the loss in episode $t$, while not straying too far from the previously chosen occupancy measure. Formally, given a parameter $\eta > 0$,
$$
q_{t+1}=\arg\min_{q\in\Delta(M,i(t))}\eta\left\langle q,z_{t}\right\rangle +D(q||q_{t})
$$
where $z_t \in \partial f^\mathcal{C}(q_t;\ell_t)$ is a sub-gradient and $D(q||q_t)$ is the unnormalized KL divergence between two occupancy measures defined as
\begin{multline*}
    D(q||q')=\sum_{x,a,x'}q(x,a,x')\ln\frac{q(x,a,x')}{q'(x,a,x')} \\ -q(x,a,x')+q'(x,a,x')    
\end{multline*}

We now proceed to show that this optimization problem can be solved efficiently. From the theory of OMD it is known that we can split this problem as follows: we start by solving the unconstrained problem, and then project the unconstrained minimizer into the feasible set, namely,
\begin{align}
    \tilde{q}_{t+1}&=\arg\min_{q}\eta\left\langle q,z_{t}\right\rangle +D(q||q_{t}) \nonumber
    \\
    \label{eq:update}
    q_{t+1}&=\arg\min_{q\in\Delta(M,i(t))}D(q||\tilde{q}_{t+1})
\end{align}

The unconstrained problem can be easily solved by setting $\tilde{q}_{t+1}(x,a,x')=q_{t}(x,a,x')e^{-\eta z_{t}(x,a,x')}$ for every $(x,a,x') \in X \times A \times X_{k(x)+1}$. Theorem \ref{th:opt} shows that the second optimization problem can be reduced to a convex optimization problem with only non-negativity constraints (and no constraints about the relations between the variables), which can be solved efficiently using iterative methods \cite{boyd}.

Before stating the theorem we consider some definitions that will simplify its formulation. Let $v:X \times A \times X \rightarrow \mathbb{R}$ be a value function and $e:X \times A \times X \rightarrow \mathbb{R}$ be an error function. We use $v$ and $e$ to define an estimated Bellman error.

\begin{definition}
For every $t=1,\dots,T$ define the estimated Bellman error for episode $t$, given value function $v$ and error function $e$, as
\begin{align*}
    B^{v,e}_t(x,a,x') 
    & =  e(x,a,x') + v(x,a,x') -\eta z_t(x,a,x')
    \\
    & \quad - \sum_{y\in X_{k(x)+1}} \bar{P}_{i(t)}(y|x,a)v(x,a,y)
\end{align*}
\end{definition}

We would like to define a parameterization to $v$ and $e$ using variables that will later be known as Lagrange multipliers. Let $\beta:X \rightarrow \mathbb{R}$ and let $\mu = (\mu^+ , \mu^-)$ such that $\mu^+, \mu^-:X \times A \times X \rightarrow \mathbb{R}_{\geq 0}$. We define the following parameterization to $v$ and $e$ using $\beta$ and $\mu$.
\begin{align*}
    v^\mu(x,a,x') & = \mu^{-}(x,a,x') - \mu^{+}(x,a,x') 
    \\
    e^{\mu,\beta}(x,a,x') & = (\mu^{+}(x,a,x') + \mu^{-}(x,a,x') ) \epsilon_{i(t)}(x,a) 
    \\
    & \quad + \beta(x') - \beta(x)
\end{align*}

Now we are ready to state the theorem.

\begin{theorem}
\label{th:opt}
Let $t>1$ and define the function
$$
    Z^{k}_{t}(v,e)=\sum_{x\in X_{k}}\sum_{a\in A}\sum_{x'\in X_{k+1}}q_{t}(x,a,x')e^{B^{v,e}_t(x,a,x')}    
$$
Then the solution to optimization problem \eqref{eq:update} is
$$
q_{t+1}(x,a,x') = \frac{q_{t}(x,a,x')e^{B_t^{v^{\mu_t},e^{\mu_t,\beta_t}}(x,a,x')}}{Z^{k(x)}_{t}(v^{\mu_t},e^{\mu_t,\beta_t})}   
$$
where
\begin{equation}
    \label{eq:v_t_mu_t}
    \beta_t , \mu_t = \arg \min_{\beta,\mu \geq 0} \sum_{k=0}^{L-1} \ln Z^{k}_{t}(v^{\mu} , e^{\mu,\beta})
\end{equation}
\end{theorem}

\begin{algorithm}
   \caption{UC-O-REPS Algorithm}
   \label{alg:uc-o-reps}
\begin{algorithmic}
   \STATE {\bfseries Input:} state space $X$, action space $A$, time horizon $T$, convexly-measurable performance criterion $\mathcal{C}$ with its criterion function $f^\mathcal{C}$, optimization parameter $\eta$ and confidence parameter $\delta$.
   \STATE
   \STATE {\bfseries Initialization:} 
   
   \STATE start first epoch:
    $
        i(1)  \leftarrow 1
        \quad ; \quad
        t_{1}  \leftarrow 1
    $
    
    \STATE initialize counters $\forall (x,a,x')$:
    \begin{align*}
        n_{1}(x,a) \leftarrow 0
        \quad &; \quad
        N_{1}(x,a) \leftarrow 0
        \\
        m_{1}(x'|x,a) \leftarrow 0
        \quad &; \quad
        M_{1}(x'|x,a) \leftarrow 0
    \end{align*}
    
    \STATE initialize first policy $\forall (x,a)$:
    $
    \pi_1(a|x) \leftarrow \frac{1}{|A|}
    $
    
    \STATE initialize first  occupancy measure $\forall k \quad \forall (x,a,x') \in X_k \times A \times X_{k+1}$:
    $
    q_1(x,a,x') \leftarrow \frac{1}{|X_k||A||X_{k+1}|}
    $
   
   \STATE
   \FOR{$t=1$ {\bfseries to} $T$}
   
   \STATE traverse trajectory $U_t$ using policy $\pi_t$
   \STATE observe loss function $\ell_{t}$
   
   \STATE update epoch counters $\forall k$:
   \begin{align*}
   n_{i(t)}(x_{k}^{(t)},a_{k}^{(t)}) &\leftarrow n_{i(t)}(x_{k}^{(t)},a_{k}^{(t)})+1
   \\
    m_{i(t)}(x_{k+1}^{(t)}|x_{k}^{(t)},a_{k}^{(t)}) &\leftarrow m_{i(t)}(x_{k+1}^{(t)}|x_{k}^{(t)},a_{k}^{(t)}) +1
   \end{align*}

   \IF{$\exists (x,a)\in X\times A.\quad n_{i(t)}(x,a)\geq N_{i(t)}(x,a)$}
   
   \STATE start new epoch:
   $$
       i(t+1)\leftarrow i(t)+1
       \quad ; \quad
       t_{i(t+1)} \leftarrow t+1
   $$
   
   \STATE initialize epoch counters $\forall (x,a,x')$:
   $$
   n_{i(t+1)}(x,a) \leftarrow 0
   \quad ; \quad
    m_{i(t+1)}(x'|x,a) \leftarrow 0
   $$
   
   \STATE update total counters $\forall(x,a,x')$:
   \begin{align*}
       N_{i(t+1)}(x,a) & \leftarrow N_{i(t)}(x,a) + n_{i(t)}(x,a)
       \\
       M_{i(t+1)}(x'|x,a) & \leftarrow M_{i(t)}(x'|x,a) + m_{i(t)}(x'|x,a)
   \end{align*}
   
   \STATE compute probability estimate $\forall(x,a,x')$:
   $$
   \bar{P}_{i(t+1)}(x'|x,a) \leftarrow \frac{M_{i(t+1)}(x'|x,a)}{\max\left\{ 1,N_{i(t+1)}(x,a)\right\} }
   $$
   \ELSE
   \STATE continue in the same epoch:
   $
   i(t+1) \leftarrow i(t)
   $
   
   \ENDIF

   \STATE compute policy for next episode:
   $$
   q_{t+1}, \pi_{t+1} \leftarrow \mbox{\tt{Comp-Policy}}(q_t,\bar{P}_{i(t+1)}, \ell_t, f^\mathcal{C})
   $$
   
   \ENDFOR
\end{algorithmic}
\end{algorithm}

\begin{algorithm}
   \caption{Comp-Policy Procedure}
   \label{alg:comp-pol-prod}
\begin{algorithmic}
    \STATE {\bfseries Input:} previous occupancy measure $q_t$, transition function estimate $\bar{P}_{i(t+1)}$, current loss function $\ell_t$ and convex criterion function $f^\mathcal{C}$.
   
   \STATE
   \STATE obtain sub-gradient $z_t \in \partial f^\mathcal{C}(q_t;\ell_t)$
   
   \STATE solve optimization problem~\eqref{eq:v_t_mu_t}:
   $$
   \beta_t , \mu_t = \arg \min_{\beta,\mu \geq 0} \sum_{k=0}^{L-1} \ln Z^{k}_{t}(v^{\mu} , e^{\mu,\beta})
   $$
   
   \STATE compute next occupancy measure $\forall (x,a,x')$:
   $$
   q_{t+1}(x,a,x') = \frac{q_{t}(x,a,x')e^{B^{v^{\mu_t},e^{\mu_t,\beta_t}}(x,a,x')}}{Z^{k(x)}_{t}(v^{\mu_t},e^{\mu_t,\beta_t})}
   $$
   
   \STATE compute next policy $\forall (x,a)$:
   $$
   \pi_{t+1}(a|x)=\frac{\sum_{x'\in X_{k(x)+1}}q_{t+1}(x,a,x')}{\sum_{b\in A}\sum_{x'\in X_{k(x)+1}}q_{t+1}(x,b,x')}
   $$
   
\end{algorithmic}
\end{algorithm}

\begin{proof}
First of all we would like to reformulate optimization problem~\eqref{eq:update} as a convex optimization problem. Notice that the target function is convex (since it is the KL-divergence) and so are constraints \eqref{eq:sum_one}, \eqref{eq:flow} of $\Delta(M,i)$ (where $i=i(t)$). As for constraint \eqref{eq:l1_dist}, we will need to write it differently.

Let $(x,a) \in X \times A$, we can replace
$$
\left\lVert \frac{q(x,a,\cdot)}{\sum_{y\in X_{k(x)+1}}q(x,a,y)} - \bar{P}_{i}(\cdot|x,a) \right\rVert_1 \leq \epsilon_{i}(x,a)
$$

with $|X_{k(x)+1}|+1$ constraints as follows. For each $x'\in X_{k(x)+1}$ we bound the difference in the transition probability with a new variable $\epsilon'(x,a,x')$ and then we bound their sum with the original bound $\epsilon_{i}(x,a)$. That is
\begin{align*}
\left| \frac{q(x,a,x')}{\sum_{y\in X_{k(x)+1}}q(x,a,y)} - \bar{P}_{i}(x'|x,a) \right| &\leq \epsilon'(x,a,x')
\\
\sum_{x' \in X_{k(x)+1}} \epsilon'(x,a,x') & \leq \epsilon_{i}(x,a)
\end{align*}

Now we can get rid of the denominator by multiplying the equation and then replacing $\epsilon'(x,a,x')$ with a different variable $\epsilon(x,a,x') = \epsilon'(x,a,x') \sum_{y\in X_{k(x)+1}} q(x,a,y)$. Moreover, we will discard the absolute value by replacing it with two linear constraints. The resulting constraints are,
\begin{align*}
q(x,a,x') - \bar{P}_{i}(x'|x,a)  \sum_{y\in X_{k(x)+1}}q(x,a,y)  &\leq \epsilon(x,a,x')
\\
\bar{P}_{i}(x'|x,a)  \sum_{y\in X_{k(x)+1}}q(x,a,y)  - q(x,a,x') &\leq  \epsilon(x,a,x')
\\
\sum_{x' \in X_{k(x)+1}} \epsilon(x,a,x') \leq \epsilon_{i}(x,a) \sum_{x' \in X_{k(x)+1}} & q(x,a,x')
\end{align*}

This gives us a convex optimization problem with linear constraints. This problem obtains strong duality because: (1) The target function is bounded from below because KL-divergence is non-negative, (2) The target function and all constraints are convex, (3) Slater condition holds (easy to check).

Thus we can use the method of Lagrange multipliers, and we are ensured that the solution we get is optimal and finite.
The full derivation can be found in the supplementary material and yields the aforementioned result.
\end{proof}

\section{Analysis}
\label{sec:reg}

In this section we bound the regret of the UC-O-REPS algorithm, by combining ideas from the regret analyses of OMD and UCRL-2. First we partition the regret into two terms: $\hat{R}_{1:T}^{APP}$ - which includes the error that comes from the estimation of the unknown transition function, and $\hat{R}_{1:T}^{ON}$ - which includes the error that comes from choosing sub-optimal policies. Formally,
\begin{align*}
    \hat{R}_{1:T}^\mathcal{C} & =\hat{L}_{1:T}^\mathcal{C}(\{\ell_t\}_{t=1}^T)-\min_{\pi}L_{1:T}^\mathcal{C}(\pi;\{\ell_t\}_{t=1}^T)
    \\
    & = \sum_{t=1}^T  \mathcal{C} (
    \mathbb{E}\left[ \ell_t(U) | P,\pi_t \right] ) -
    \sum_{t=1}^T \mathcal{C} ( \mathbb{E}\left[ \ell_t(U) | P,\pi \right] )
    \\
    & = \left( \sum_{t=1}^T  \mathcal{C} (
    \mathbb{E}\left[ \ell_t(U) | P,\pi_t \right] ) -  \mathcal{C} ( \mathbb{E}\left[ \ell_t(U) | P_t,\pi_t \right] )
    \right)
    \\
    & \quad + \left( \sum_{t=1}^T  \mathcal{C} (
    \mathbb{E}\left[ \ell_t(U) | P_t,\pi_t \right] ) -  \mathcal{C} ( \mathbb{E}\left[ \ell_t(U) | P,\pi \right] )
    \right)
    \\
    & \stackrel{def}{=} \hat{R}_{1:T}^{APP} + \hat{R}_{1:T}^{ON} 
\end{align*}
where $P_t = P^{q_t}$ and $\pi_t = \pi^{q_t}$.

Notice that $\mathcal{C} ( \mathbb{E}\left[ \ell_t(U) | P_t,\pi_t \right] ) = f^\mathcal{C}(q_t;\ell_t)$ but it isn't the case with $\mathcal{C} ( \mathbb{E}\left[ \ell_t(U) | P,\pi_t \right] )$ because $q_t$ is not necessarily an occupancy measure of $M$.
Theorems \ref{th:confid} and \ref{th:online} bound each of  these terms, which yields our main result.

\begin{theorem}
\label{th:reg}
Let $M= \left( X,A,P, \{ \ell_t \}_{t=1}^T \right)$ be an episodic loop-free adversarial MDP, and let $\mathcal{C}$ be a convexly-measurable performance criterion such that $f^\mathcal{C}$ is $F$-Lipschitz. Then, with probability at least $1 - 2 \delta$, UC-O-REPS with $\eta = \sqrt{\frac{\ln\frac{|X|^{2}|A|}{L^{2}}}{F^2T}}$ achieves the following regret,
$$
\hat{R}_{1:T}^\mathcal{C} \leq 15 F L|X| \sqrt{T |A|\ln{\frac{T|X||A|}{\delta}}} 
$$
\end{theorem}

An immediate corollary of this theorem is the regret bound in the classical case of total expected loss performance criterion.

\begin{corollary}
Running UC-O-REPS in an episodic loop-free adversarial MDP $M= \left( X,A,P, \{ \ell_t \}_{t=1}^T \right)$ yields the following regret with respect to the total expected loss, when setting $\delta = \frac{|X||A|}{T}$,
$$
\hat{R}_{1:T}^{\mathcal{C}^{TEL}} \leq 25 L|X| \sqrt{T|A|\ln{T}}
$$
\end{corollary}

\begin{proof}
For the total expected loss performance criterion we have that $f^{\mathcal{C}^{TEL}}(q_t;\ell_t)=\left\langle q_{t},\ell_{t}\right\rangle$ and  therefore the gradient of $f^\mathcal{C}$ is $z_t=\ell_t$. Since the losses are bounded by $1$, we have that $f^\mathcal{C}$ is $1$-Lipschitz, i.e.,  $F=1$.

Recall that in this case the regret is an expectation. With probability at least $1 - 2\delta$ it is bounded using Theorem \ref{th:reg}, and with probability at most $2 \delta$ we have a worst case bound of $TL$. Substituting $\delta$ and using the law of total expectation finishes the proof.
\end{proof}

\subsection{Bounding $\hat{R}_{1:T}^{APP}$}

The term $\hat{R}_{1:T}^{APP}$ is a result of the learner's lack of knowledge about the environment's dynamics. Since the dynamics are stochastic the learner estimates the transition probabilities to build confidence sets. It then selects occupancy measures from within these confidence sets, but they are not exactly occupancy measures of $M$.

In this section we bound the difference between the loss of the learner's chosen policies in $M$ and the loss of these policies in the ``optimistic'' MDP (the one induced by the occupancy measure $q_t$), i.e.,
$$
\hat{R}_{1:T}^{APP} = \sum_{t=1}^T \mathcal{C} (
\mathbb{E}\left[ \ell_t(U) | P,\pi_t \right] ) -  \mathcal{C} ( \mathbb{E}\left[ \ell_t(U) | P_t,\pi_t \right] )
$$

The way the algorithm minimizes this difference is through shrinking of the confidence sets.
The following bound on $\hat{R}_{1:T}^{APP}$ is adapted from arguments in the regret analysis of UCRL-2, and the proof can be found in the supplementary material.

\begin{theorem}
\label{th:confid}
Let $M= \left( X,A,P, \{ \ell_t \}_{t=1}^T \right)$ be an episodic loop-free adversarial MDP, and let $\mathcal{C}$ be a convexly-measurable performance criterion such that $f^\mathcal{C}$ is $F$-Lipschitz. Then, with probability at least $1 - 2 \delta$, UC-O-REPS obtains,
$$
\hat{R}_{1:T}^{APP} \leq 3 FL|X| \left(2 \sqrt{T\ln{\frac{L}{\delta}}} +  3\sqrt{T|A| \ln{\frac{T|X||A|} {\delta}}} \right)
$$
\end{theorem}

\subsection{Bounding $\hat{R}_{1:T}^{ON}$}

The term $\hat{R}_{1:T}^{ON}$ is a result of the learner's lack of knowledge about the loss functions. Since the sequence of loss functions can be arbitrary, the learner handles it with tools from online convex optimization.

In this section we ignore the fact that the occupancy measures chosen by the learner are not exactly occupancy measures of $M$, since this issue was already addressed in the previous section bounding $\hat{R}_{1:T}^{APP}$.
Here we are only interested in the following difference
$$
\hat{R}_{1:T}^{ON} = \sum_{t=1}^T \mathcal{C} (
\mathbb{E}\left[ \ell_t(U) | P_t,\pi_t \right] ) -  \mathcal{C} ( \mathbb{E}\left[ \ell_t(U) | P,\pi \right] )
$$
First we use the connection between $\mathcal{C}$ and $f^\mathcal{C}$, and the convexity of $f^\mathcal{C}$ to obtain 
$$
\hat{R}_{1:T}^{ON}  = \sum_{t=1}^T f^\mathcal{C}(q_t;\ell_t) - f^\mathcal{C}(q;\ell_t)
\leq \sum_{t=1}^T \left\langle q_{t}-q,z_{t}\right\rangle
$$
where $z_t \in \partial f^\mathcal{C}(q_t;\ell_t)$.

Now we can use arguments from online linear optimization. Specifically, the following theorem is an adaptation of OMD regret analysis to our setting.

\begin{theorem}
\label{th:online}
Let $M= \left( X,A,P, \{ \ell_t \}_{t=1}^T \right)$ be an episodic loop-free adversarial MDP, and let $\mathcal{C}$ be a convexly-measurable performance criterion such that $f^\mathcal{C}$ is $F$-Lipschitz.
Then, with probability at least $1-\delta$, UC-O-REPS obtains the following for every $q\in \Delta(M)$.
$$
\hat{R}_{1:T}^{ON} \leq \sum_{t=1}^{T}\left\langle q_{t}-q,z_{t}\right\rangle \leq \eta F^2LT +\frac{L\ln\frac{|X|^{2}|A|}{L^{2}}}{\eta}
$$
and setting $\eta = \sqrt{\frac{\ln\frac{|X|^{2}|A|}{L^{2}}}{F^2T}}$ yields
$$
\hat{R}_{1:T}^{ON} \leq 2FL\sqrt{2T\ln{\frac{|X||A|}{L}}}
$$
where $q_t$ is the occupancy measure chosen by UC-O-REPS in episode $t$, and $z_t \in \partial f^\mathcal{C}(q_t;\ell_t)$.
\end{theorem}

\begin{proof}
By standard arguments of OMD regret analysis (the full proof can be found in the full version of the paper) we have that
$$
\sum_{t=1}^{T}\left\langle q_{t}-q,z_{t}\right\rangle \leq\sum_{t=1}^{T}\left\langle q_{t}-\tilde{q}_{t+1},z_{t}\right\rangle +\frac{D(q||q_{1})}{\eta}
$$

However these arguments assume that $q_t$ are chosen from within $\Delta(M)$ so we need to show that they are still valid.
From Lemma \ref{lem:high_prob} we know that $\Delta(M) \subseteq \Delta(M,i)$ for every $i$ with probability at least $1-\delta$. Therefore, by choosing approximate occupancy measures we can only improve the regret so the arguments are indeed valid.

Using the exact form of $\tilde{q}_{t+1}$ and the fact that $e^{x}\geq1+x$, we get that
$$
\tilde{q}_{t+1}(x,a,x')\geq q_{t}(x,a,x')-\eta q_{t}(x,a,x') z_{t}(x,a,x')
$$
and therefore
\begin{align*}
    \sum_{t=1}^{T} \left\langle q_{t}-\tilde{q}_{t+1},z_{t}\right\rangle  &\leq \eta \sum_{t=1}^{T} \sum_{x,a,x'}q_{t}(x,a,x') z_{t}^{2}(x,a,x')
    \\
    & \leq \eta F^2 \sum_{t=1}^{T} \sum_{x,a,x'} q_{t}(x,a,x') = \eta F^2 L T
\end{align*}

For the second term, $D(q||q_{1})/\eta$, we use the fact that the unnormalized KL divergence is the Bregman divergence associated with the unnormalized negative entropy, defined as follows.
$$
R(q) = \sum_{x,a,x'} q(x,a,x')\ln{q(x,a,x')} - q(x,a,x')
$$

Now from standard arguments we obtain
\begin{align*}
    D(q||q_{1})&\leq R(q)-R(q_{1})
    \\
    & \leq \sum_{x\in X} \sum_{a\in A}\sum_{x'\in X_{k(x)+1}}q_{1}(x,a,x')\ln\frac{1}{q_{1}(x,a,x')}
    \\
    & \leq\sum_{k=0}^{L-1}\ln|X_{k}||A||X_{k+1}|\leq L\ln\frac{|X|^{2}|A|}{L^{2}}
\end{align*}

Putting these two bounds together completes the proof.
\end{proof}

\section{Conclusions and Future Work}

In this paper we considered online learning in adversarial MDPs where the transition function is not known to the learner and the losses can change arbitrarily between episodes, and showed an algorithm that achieves $\tilde{O}(L|X|\sqrt{T|A|})$ regret. The algorithm is based on a combination of the OMD method for online convex optimization, and the UCRL-2 algorithm for reinforcement learning. Moreover, we extended the adversarial MDP model to include convex performance criteria, and showed that our algorithm achieves near-optimal regret bounds in this model as well.

The natural open problem is whether the lower bound of $\Omega(\sqrt{L|X||A|T})$ \cite{jaksch} can be achieved in this model. An algorithm that achieves this will have to build upon a different method than UCRL-2, and it will be interesting to see if the techniques of \citet{azar} can be implemented here.
Another interesting open question is to consider bandit feedback when the transition function is unknown. This question seems to be difficult because the natural approach of building an unbiased estimator for the losses cannot be implemented easily, since the natural construction of inverse probability estimator requires knowledge of the transition probabilities.

\newpage
\section*{Acknowledgements}
This work was supported in part by a grant
from the Israel Science Foundation (ISF) and by the Tel Aviv University Yandex Initiative in Machine Learning.

\bibliography{example_paper}
\bibliographystyle{icml2019}

\clearpage
\onecolumn
\input{supp}
\end{document}

%% file: supp.tex
\appendix

\section{Proof of Theorem 4.2 Cont.}

In the proof of Theorem 4.2 we showed that the following optimization problem
$$
q_{t+1} = \arg \min_{q \in \Delta(M,i(t))} D(q||\tilde{q}_{t+1}) 
$$
can be reformulated as the following convex optimization problem ($i=i(t)$):
\begin{align*}
    \min_{q,\epsilon} & D(q||\tilde{q}_{t+1})
    \\
    s.t. & \sum_{x\in X_{k}}\sum_{a\in A}\sum_{x'\in X_{k+1}}q(x,a,x') = 1
    & \forall k=0,\dots, L-1
    \\
    & \sum_{x'\in X_{k+1}}\sum_{a\in A}q(x,a,x') = \sum_{x'\in X_{k-1}}\sum_{a\in A}q(x',a,x) & \forall k = 1,\dots,L-1 \quad \forall x\in X_k
    \\
    & q(x,a,x') - \bar{P}_{i}(x'|x,a) \sum_{y\in X_{k+1}}q(x,a,y) \leq \epsilon(x,a,x') & \forall k = 0,\dots,L-1 \quad \forall (x,a,x')\in X_k \times A \times X_{k+1}
    \\
    & \bar{P}_{i}(x'|x,a) \sum_{y\in X_{k+1}}q(x,a,y) -q(x,a,x') \leq \epsilon(x,a,x') & \forall k = 0,\dots,L-1 \quad \forall (x,a,x')\in X_k \times A \times X_{k+1}
    \\
    & \sum_{x' \in X_{k+1}} \epsilon(x,a,x')  \leq \epsilon_{i}(x,a) \sum_{x' \in X_{k+1}} q(x,a,x') & \forall k = 0,\dots,L-1 \quad \forall (x,a)\in X_k \times A
    \\
    & q(x,a,x') \geq 0 & \forall k = 0,\dots,L-1 \quad \forall (x,a,x')\in X_k \times A \times X_{k+1}
\end{align*}

Now we will derive the solution to this problem using Lagrange multipliers.
First we write the Lagrangian with $\lambda,\beta,\mu,\mu^{+},\mu^{-}$ as Lagrange multipliers. Notice that we omit the non-negativity constraints, which we can justify since the solution will be non-negative anyway.
\begin{align*}
    \mathcal{L} (q,\epsilon) &= D(q||\tilde{q}_{t+1}) + \sum_{k=0}^{L-1} \lambda_k \left( \sum_{x\in X_k} \sum_{a \in A} \sum_{x' \in X_{k+1}} q(x,a,x')-1 \right)
    \\
    & + \sum_{k=1}^{L-1} \sum_{x \in X_k} \beta(x) \left( \sum_{a \in A} \sum_{x' \in X_{k+1}} q(x,a,x') - \sum_{a \in A} \sum_{x' \in X_{k-1}} q(x',a,x) \right)
    \\
    & + \sum_{k=0}^{L-1} \sum_{x \in X_k} \sum_{a \in A} \sum_{x' \in X_{k+1}} \mu^{+}(x,a,x
    ') \left( q(x,a,x') - \bar{P}_i(x'|x,a) \sum_{y \in X_{k+1}} q(x,a,y) - \epsilon(x,a,x') \right)
    \\
    & + \sum_{k=0}^{L-1} \sum_{x \in X_k} \sum_{a \in A} \sum_{x' \in X_{k+1}} \mu^{-}(x,a,x
    ') \left( \bar{P}_i(x'|x,a) \sum_{y \in X_{k+1}} q(x,a,y) - q(x,a,x')  - \epsilon(x,a,x') \right)
    \\
    & + \sum_{k=0}^{L-1} \sum_{x \in X_k} \sum_{a \in A} \mu(x,a) \left( 
    \sum_{x' \in X_{k+1}} \epsilon(x,a,x') - \epsilon_i(x,a) \sum_{x' \in X_{k+1}} q(x,a,x')\right)
\end{align*}

Let $(x,a,x') \in X \times A \times X_{k(x)+1}$ and consider the derivative with respect to $\epsilon(x,a,x')$.
$$
\frac{\partial \mathcal{L}}{\partial \epsilon(x,a,x')}  = - \mu^{+}(x,a,x') - \mu^{-}(x,a,x') + \mu(x,a)
$$
So setting the gradient to zero we obtain
$$
\mu(x,a) = \mu^{+}(x,a,x') + \mu^{-}(x,a,x')
$$
Thus, we can discard $\mu(x,a)$ to obtain an equivalent Lagrangian. Notice that this way we also get rid of the  $\epsilon(x,a,x')$ variables.
\begin{align*}
    \mathcal{L} (q) &= D(q||\tilde{q}_{t+1}) + \sum_{k=0}^{L-1} \lambda_k \left( \sum_{x\in X_k} \sum_{a \in A} \sum_{x' \in X_{k+1}} q(x,a,x')-1 \right)
    \\
    & + \sum_{k=1}^{L-1} \sum_{x \in X_k} \beta(x) \left( \sum_{a \in A} \sum_{x' \in X_{k+1}} q(x,a,x') - \sum_{a \in A} \sum_{x' \in X_{k-1}} q(x',a,x) \right)
    \\
    & + \sum_{k=0}^{L-1} \sum_{x \in X_k} \sum_{a \in A} \sum_{x' \in X_{k+1}} \mu^{+}(x,a,x
    ') \left( (1-\epsilon_i(x,a))q(x,a,x') - \bar{P}_i(x'|x,a) \sum_{y \in X_{k+1}} q(x,a,y) \right)
    \\
    & + \sum_{k=0}^{L-1} \sum_{x \in X_k} \sum_{a \in A} \sum_{x' \in X_{k+1}} \mu^{-}(x,a,x
    ') \left( \bar{P}_i(x'|x,a) \sum_{y \in X_{k+1}} q(x,a,y) - (1+\epsilon_i(x,a))q(x,a,x') \right)
\end{align*}

Now we consider the derivative with respect to $q(x,a,x')$. We denote $\beta(x_0)=\beta(x_L)=0$ to avoid addressing the edge cases explicitly.
\begin{align*}
    \frac{\partial \mathcal{L}}{\partial q(x,a,x')} & = \ln{q(x,a,x')} - \ln{\tilde{q}_{t+1}(x,a,x')} + \lambda_k + \beta(x) - \beta(x')
    \\
    & +(1-\epsilon_i(x,a)) \mu^{+}(x,a,x') -(1+\epsilon_i(x,a)) \mu^{-}(x,a,x')
    \\
    & + \sum_{y \in X_{k(x)+1}} \bar{P}_i(y|x,a) (\mu^{-}(x,a,y) - \mu^{+}(x,a,y))
\end{align*}

We define the following value function $v$ and error function $e$ parameterized by $\mu$ and $\beta$, and an estimated Bellman error.
\begin{align*}
    v^\mu(x,a,x') & = \mu^{-}(x,a,x') - \mu^{+}(x,a,x') 
    \\
    e^{\mu,\beta}(x,a,x') & = (\mu^{+}(x,a,x') + \mu^{-}(x,a,x') ) \epsilon_{i}(x,a) + \beta(x') - \beta(x)
    \\
    B^{v,e}_t(x,a,x') 
    & =  e(x,a,x') + v(x,a,x') -\eta z_t(x,a,x') - \sum_{y\in X_{k(x)+1}} \bar{P}_{i}(y|x,a)v(x,a,y)
\end{align*}
So the derivative becomes
\begin{align*}
    \frac{\partial \mathcal{L}}{\partial q(x,a,x')} & = \ln{\frac{q(x,a,x')}{\tilde{q}_{t+1}(x,a,x')}} + \lambda_k - e^{\mu,\beta}(x,a,x') - v^\mu(x,a,x') + \sum_{y \in X_{k(x)+1}} \bar{P}_i(y|x,a) v^\mu(x,a,y)
    \\
    & = \ln{q(x,a,x')} - \ln{\tilde{q}_{t+1}(x,a,x')} + \lambda_k -\eta z_t(x,a,x') - B^{v^\mu,e^{\mu,\beta}}_t(x,a,x') 
\end{align*}

Setting the gradient to zero and using the explicit form of $\tilde{q}_{t+1}(x,a,x')$ we obtain
\begin{align*}
    q_{t+1}(x,a,x') & = \tilde{q}_{t+1}(x,a,x')  e^{-\lambda_{k} + \eta z_t(x,a,x') +  B^{v^\mu,e^{\mu,\beta}}_t(x,a,x')}
    \\
    & = q_t(x,a,x')e^{ -\eta z_t(x,a,x')} e^{-\lambda_{k} + \eta z_t(x,a,x') +  B^{v^\mu,e^{\mu,\beta}}_t(x,a,x')}
    \\
    & = q_t(x,a,x') e^{-\lambda_{k} +   B^{v^\mu,e^{\mu,\beta}}_t(x,a,x')}
\end{align*}

We can use the first constraint to discover that $\lambda_k$ is a normalizer for every $k=0,\dots,L-1$, i.e.
\begin{align*}
    1 & = \sum_{x\in X_k} \sum_{a \in A} \sum_{x' \in X_{k+1}} q_{t+1}(x,a,x')
    \\
    1 & = \sum_{x\in X_k} \sum_{a \in A} \sum_{x' \in X_{k+1}} q_t(x,a,x') e^{-\lambda_{k} +   B^{v^\mu,e^{\mu,\beta}}_t(x,a,x')}
    \\
    e^{\lambda_k} & =   \sum_{x\in X_k} \sum_{a \in A} \sum_{x' \in X_{k+1}} q_t(x,a,x') e^{  B^{v^\mu,e^{\mu,\beta}}_t(x,a,x')}
\end{align*}
so defining $Z^{k}_{t}(v,e)=\sum_{x\in X_{k}}\sum_{a\in A}\sum_{x'\in X_{k+1}}q_{t}(x,a,x')e^{B^{v,e}_t(x,a,x')}$ , we obtain
$$
q_{t+1}(x,a,x') = \frac{q_{t}(x,a,x')e^{B^{v^{\mu},e^{\mu,\beta}}(x,a,x')}}{Z^{k(x)}_{t}(v^{\mu},e^{\mu,\beta})}   
$$

Now to find $\beta$ and $\mu$ we consider the dual problem. Substituting $q_{t+1}$ back into $\mathcal{L}$ we obtain the following dual problem.
$$
\max_{\beta,\mu \geq 0} \min_q \mathcal{L}(q) =
\max_{\beta,\mu \geq 0} \mathcal{L}(q_{t+1}) = 
\max_{\beta,\mu \geq 0} - \sum_{k=0}^{L-1} \ln{Z^{k}_{t}(v^{\mu},e^{\mu,\beta})} -1 + \sum_{x,a,x'} \tilde{q}_{t+1}(x,a,x')
$$

So after ignoring constants we observe that
$$
\beta_t,\mu_t = \arg \min_{\beta,\mu \geq 0} \sum_{k=0}^{L-1} \ln{Z^{k}_{t}(v^{\mu},e^{\mu,\beta})}
$$

\section{Proof of Theorem 5.2}

First we reduce bounding $\hat{R}_{1:T}^{APP}$ to bounding the $L_1$-distance between $q^{P_t,\pi_t}$ and $q^{P,\pi_t}$, where $P_t=P^{q_t}$ and $\pi_t=\pi^{q_t}$.
\begin{align}
    \nonumber
    \hat{R}_{1:T}^{APP}
    & = \sum_{t=1}^T \mathcal{C} (
    \mathbb{E}\left[ \ell_t(U) | P,\pi_t \right] ) -  \mathcal{C} ( \mathbb{E}\left[ \ell_t(U) | P_t,\pi_t \right] )
    \\
    \nonumber
    & = \sum_{t=1}^T f^\mathcal{C}(q^{P,\pi_t};\ell_t) - f^\mathcal{C}(q^{P_t,\pi_t};\ell_t)
    \\
    \label{eq:grad}
    & \leq \sum_{t=1}^T \left< \bar{z}_t , q^{P,\pi_t} - q^{P_t,\pi_t} \right>
    \\
    \label{eq:holder}
    & \leq \sum_{t=1}^T \left\lVert \bar{z}_t \right\rVert_\infty \left\lVert q^{P,\pi_t} - q^{P_t,\pi_t} \right\rVert_1
    \\
    \label{eq:F-bound}
    & \leq F \sum_{t=1}^T \left\lVert q^{P,\pi_t} - q^{P_t,\pi_t} \right\rVert_1
\end{align}
where $\bar{z}_t \in \partial f^\mathcal{C}(q^{P,\pi_t};\ell_t)$ and \eqref{eq:grad} follows from the definition of the sub-gradient, \eqref{eq:holder} follows from H\"older's inequality, and \eqref{eq:F-bound} follows because $f^\mathcal{C}$ is $F$-Lipschitz.

Therefore, We are left with bounding $\sum_{t=1}^T \left\lVert q^{P,\pi_t} - q^{P_t,\pi_t} \right\rVert_1$. From now on, we follow arguments from the regret analysis of UCRL-2, since we just need to bound the distance between occupancy measures that are in the confidence sets, and the performance criterion is not involved anymore.

We introduce some new notations that will simplify some equations. we denote the probability to visit a state-action pair $(x,a)$ (or a state $x$) under occupancy measure $q$ as $q(x,a)$ (or $q(x)$), i.e.,
\begin{align*}
    q(x,a) & = \sum_{x' \in X_{k(x)+1}} q(x,a,x')
    \\
    q(x) & = \sum_{a \in A} q(x,a)
\end{align*}
In addition, for every $(x,a) \in X \times A$ and every $t=1,\dots,T$, denote $\xi_t(x,a) = \left\lVert P_t(\cdot | x,a) - P(\cdot | x,a) \right\rVert_1$.

Now we show how to use these notations to bound the aforementioned $L_1$-distance.

\begin{lemma}
Let $\{ \pi_t \}_{t=1}^T$ be policies and let $\{ P_t \}_{t=1}^T$ be transition functions. Then,
\begin{equation}
\label{eq:b1}
\sum_{t=1}^T \left\lVert q^{P_t,\pi_t} - q^{P,\pi_t} \right\rVert_1 \leq \sum_{t=1}^T \sum_{x\in X} \sum_{a \in A} | q^{P_t,\pi_t}(x,a) - q^{P,\pi_t}(x,a)| + \sum_{t=1}^T  \sum_{x\in X} \sum_{a \in A} q^{P,\pi_t}(x,a) \xi_t(x,a)
\end{equation}
\end{lemma}

\begin{proof}
For every $(x,a) \in X \times A$ it holds that
\begin{align*}
    \sum_{x' \in X_{k(x)+1}} |q^{P_t,\pi_t}(x,a,x') - q^{P,\pi_t}(x,a,x')| & = \sum_{x' \in X_{k(x)+1}} |q^{P_t,\pi_t}(x,a)P_t(x'|x,a) - q^{P,\pi_t}(x,a)P(x'|x,a)|
    \\
    & \leq \sum_{x' \in X_{k(x)+1}} |q^{P_t,\pi_t}(x,a)P_t(x'|x,a) - q^{P,\pi_t}(x,a)P_t(x'|x,a)|
    \\
    & \qquad \qquad \quad + |q^{P,\pi_t}(x,a)P_t(x'|x,a) - q^{P,\pi_t}(x,a)P(x'|x,a)|
    \\
    & =  \sum_{x' \in X_{k(x)+1}} |q^{P_t,\pi_t}(x,a) - q^{P,\pi_t}(x,a)| P_t(x'|x,a)
    \\
    & \qquad \qquad \quad + |P_t(x'|x,a) - P(x'|x,a)|q^{P,\pi_t}(x,a)
    \\
    & = |q^{P_t,\pi_t}(x,a) - q^{P,\pi_t}(x,a)| + q^{P,\pi_t}(x,a) \xi_t(x,a)
\end{align*}
Summing this for all $t=1,\dots,T$ and all $(x,a) \in X \times A$ gives the result.
\end{proof}

Thus, we need to bound each of the terms on the right hand side of \eqref{eq:b1}. First, we show how to bound the first term on the right hand side of \eqref{eq:b1} using the second term.

\begin{lemma}
Let $\{ \pi_t \}_{t=1}^T$ be policies and let $\{ P_t \}_{t=1}^T$ be transition functions. Then, for every $k=1,\dots,L-1$ and every $t=1,\dots,T$, it holds that
$$
\sum_{x_k \in X_k} \sum_{a_k \in A} |q^{P_t,\pi_t}(x_k,a_k) - q^{P,\pi_t}(x_k,a_k)| \leq \sum_{s=0}^{k-1} \sum_{x_s \in X_s} \sum_{a_s\in A} q^{P,\pi_t}(x_s,a_s) \xi_t(x_s,a_s)
$$
\end{lemma}

\begin{proof}
We prove the statement by induction on $k$. For $k=1$ we have
\begin{align*}
    \sum_{x_1 \in X_1} \sum_{a_1 \in A}
    | & q^{P_t,\pi_t}(x_1,a_1)  - q^{P,\pi_t}(x_1,a_1)| =
    \\
    & = \sum_{a_0 \in A} \sum_{x_1 \in X_1} \sum_{a_1 \in A} | \pi_t(a_0|x_0)P_t(x_1|x_0,a_0)\pi_t(a_1|x_1) - \pi_t(a_0|x_0)P(x_1|x_0,a_0)\pi_t(a_1|x_1)|
    \\
     & = \sum_{a_0 \in A} \pi_t(a_0|x_0) \sum_{x_1 \in X_1} | P_t(x_1|x_0,a_0) - P(x_1|x_0,a_0)| \sum_{a_1 \in A} \pi_t(a_1|x_1)
    \\
    & \leq \sum_{a_0 \in A} \pi_t(a_0|x_0) \xi_t(x_0,a_0)
    \\
    & = \sum_{a_0 \in A} q^{P,\pi_t}(x_0,a_0) \xi_t(x_0,a_0)
\end{align*}

Now assume that the statement holds for some $k-1$. We have
\begin{align*}
    & \sum_{x_k \in X_k} \sum_{a_k \in A}
    |  q^{P_t,\pi_t}(x_k,a_k)  - q^{P,\pi_t}(x_k,a_k)| =
    \\
    & = \sum_{x_{k-1}} \sum_{a_{k-1}} \sum_{x_k} \sum_{a_k} |q^{P_t,\pi_t}(x_{k-1},a_{k-1})P_t(x_k|x_{k-1},a_{k-1})  -
    q^{P,\pi_t}(x_{k-1},a_{k-1})P(x_k|x_{k-1},a_{k-1})|\pi_t(a_k|x_k)
    \\
    & = \sum_{x_{k-1}} \sum_{a_{k-1}} \sum_{x_k} |q^{P_t,\pi_t}(x_{k-1},a_{k-1})P_t(x_k|x_{k-1},a_{k-1})  -
    q^{P,\pi_t}(x_{k-1},a_{k-1})P(x_k|x_{k-1},a_{k-1})|
    \\
    & \leq \sum_{x_{k-1}} \sum_{a_{k-1}} \sum_{x_k}
    |q^{P_t,\pi_t}(x_{k-1},a_{k-1})P_t(x_k|x_{k-1},a_{k-1})  -
    q^{P,\pi_t}(x_{k-1},a_{k-1})P_t(x_k|x_{k-1},a_{k-1})|
    \\
    & \qquad \qquad \qquad \qquad + |q^{P,\pi_t}(x_{k-1},a_{k-1})P_t(x_k|x_{k-1},a_{k-1})  -
    q^{P,\pi_t}(x_{k-1},a_{k-1})P(x_k|x_{k-1},a_{k-1})|
    \\
    & \leq \sum_{x_{k-1}} \sum_{a_{k-1}} 
    |q^{P_t,\pi_t}(x_{k-1},a_{k-1})  -
    q^{P,\pi_t}(x_{k-1},a_{k-1})| + \sum_{x_{k-1}} \sum_{a_{k-1}} q^{P,\pi_t}(x_{k-1},a_{k-1}) \xi_t(x_{k-1},a_{k-1})
\end{align*}

Finally, we use the induction hypothesis to obtain
\begin{align*}
    \sum_{x_k \in X_k} & \sum_{a_k \in A}
    |  q^{P_t,\pi_t}(x_k,a_k)  - q^{P,\pi_t}(x_k,a_k)| \leq
    \\
    & \leq \sum_{s=0}^{k-2} \sum_{x_s \in X_s} \sum_{a_s\in A} q^{P,\pi_t}(x_s,a_s) \xi_t(x_s,a_s) + \sum_{x_{k-1} \in X_{k-1}} \sum_{a_{k-1} \in A} q^{P,\pi_t}(x_{k-1},a_{k-1}) \xi_t(x_{k-1},a_{k-1})
    \\
    & = \sum_{s=0}^{k-1} \sum_{x_s \in X_s} \sum_{a_s\in A} q^{P,\pi_t}(x_s,a_s) \xi_t(x_s,a_s)
\end{align*}
\end{proof}

The following lemma will show how to bound the second term on the right hand side of \eqref{eq:b1}, and therefore obtain the bound on $\hat{R}_{1:T}^{APP}$. The proof follows the proof of Lemma 5 in Neu et al. (2012).

\begin{lemma}
Let $\{ \pi_t \}_{t=1}^T$ be policies and let $\{ P_t \}_{t=1}^T$ be transition functions such that $q^{P_t,\pi_t} \in \Delta(M,i(t))$ for every $t$. Then, with probability at least $1- 2 \delta$,
$$
\sum_{t=1}^T
\sum_{k=0}^{L-1}
\sum_{s=0}^{k-1} \sum_{x_s \in X_s} \sum_{a_s\in A} q^{P,\pi_t}(x_s,a_s) \xi_t(x_s,a_s) \leq 2
L|X|\sqrt{2T\ln{\frac{L}{\delta}}} + 3L|X|\sqrt{2T|A|\ln{\frac{T|X||A|}{\delta}}}
$$
\end{lemma}

\begin{proof}
We start by some arguments from the regret analysis of UCRL-2 (Auer et al., 2008). Let $n_i(x,a)$ be the number of times state-action pair $(x,a)$ has been visited in epoch $E_i$. Therefore, we have 
$$
N_i(x,a) = \sum_{j=1}^{i-1} n_j(x,a)
$$
We denote by $m$ the number of epochs, and by Auer et al. (2008), we have
$$
\sum_{i=1}^m \frac{n_i(x,a)}{\sqrt{N_i(x,a)}} \leq
3\sqrt{N_m(x,a)}
$$
Now by Jensen's inequality,
$$
\sum_{x \in X} \sum_{a\in A} \sum_{i=1}^m \frac{n_i(x,a)}{\sqrt{N_i(x,a)}} \leq
3\sqrt{|X||A|T}
$$

Fix arbitrary $1 \leq t \leq T$ and $0 \leq k \leq L-1$. We have
\begin{align}
    \label{eq:split}
    & \sum_{s=0}^{k-1} \sum_{x_s \in X_s} \sum_{a_s\in A} q^{P,\pi_t}(x_s,a_s) \xi_t(x_s,a_s) \leq
    \\
    \nonumber
    & \leq \sum_{s=0}^{k-1} \xi_t(x^{(t)}_s,a^{(t)}_s) + \sum_{s=0}^{k-1} \sum_{x_s \in X_s} \sum_{a_s\in A}
    \left( q^{P,\pi_t}(x_s,a_s) - \mathbb{I}\{ x^{(t)}_s=x_s,a^{(t)}_s=a_s \} \right) \xi_t(x_s,a_s)
\end{align}

Now, by Lemma 4.1, we have with probability at least $1 - \delta$ simultaneously for all $s$ that
\begin{align*}
    \sum_{t=1}^T  \xi_t(x^{(t)}_s,a^{(t)}_s) & \leq \sum_{t=1}^T \sqrt{\frac{2|X_{s+1}|\ln{\frac{T|X||A|}{\delta}}}{\max\{1,N_{i(t)}(x_s^{(t)},a_s^{(t)})\}}}
    \\
    & \leq \sum_{x_s \in X_s} \sum_{a_s\in A} \sum_{i=1}^m n_i(x_s,a_s) \sqrt{\frac{2|X_{s+1}|\ln{\frac{T|X||A|}{\delta}}}{\max\{1,N_i(x_s,a_s)\}}}
    \\
    & \leq 3\sqrt{2 T |X_s| |X_{s+1}| |A| \ln{\frac{T|X||A|}{\delta}}}
\end{align*}

For the second term on the right hand side of \eqref{eq:split}, notice that $\left( q^{P,\pi_t}(x_s) - \mathbb{I}\{ x^{(t)}_s=x_s \} \right)$ form a martingale difference sequence with respect to $\{ U_t \}_{t=1}^T$ and thus by Hoeffding-Azuma inequality and $\xi_t(x,a) \leq 2$, we have
\begin{align*}
    \sum_{t=1}^T \sum_{a_s \in A} \Bigl(  q^{P,\pi_t}(x_s,a_s) - & \mathbb{I}\{ x^{(t)}_s=x_s,a^{(t)}_s=a_s \} \Bigl) \xi_t(x_s,a_s) \leq
    \\
    & \leq 2 \sum_{t=1}^T  \left( \sum_{a_s \in A} q^{P,\pi_t}(x_s,a_s) - \sum_{a_s \in A} \mathbb{I}\{ x^{(t)}_s=x_s,a^{(t)}_s=a_s \} \right)
    \\
    & =  2 \sum_{t=1}^T  \left(  q^{P,\pi_t}(x_s) - \mathbb{I}\{ x^{(t)}_s=x_s \} \right)
    \\
    & \leq 2 \sqrt{2T\ln{\frac{L}{\delta}}}
\end{align*}
with probability at least $1 - \delta/L$. Putting everything together, the union bound implies that we have, with probability at least $1 - 2 \delta$ simultaneously for all $k=1,\dots,L-1$,
\begin{align*}
    \sum_{t=1}^T
    \sum_{s=0}^{k-1} \sum_{x_s \in X_s} \sum_{a_s\in A} q^{P,\pi_t}(x_s,a_s) \xi_t(x_s,a_s) & \leq \sum_{s=0}^{k-1} 3\sqrt{2 T |X_s| |X_{s+1}| |A| \ln{\frac{T|X||A|}{\delta}}}
    +\sum_{s=0}^{k-1}
    2 |X_s|  \sqrt{2T\ln{\frac{L}{\delta}}}
    \\
    & \leq 3L \sum_{s=0}^{k-1} \frac{1}{L} \sqrt{2 T |X_s| |X_{s+1}| |A| \ln{\frac{T|X||A|}{\delta}}} + \sum_{s=0}^{k-1}
    2 |X_s| \sqrt{2T\ln{\frac{L}{\delta}}}
    \\
    & \leq 3L \sqrt{2 T |A| \left( \frac{|X|}{L} \right)^2 \ln{\frac{T|X||A|}{\delta}}} + 2 |X|\sqrt{2T\ln{\frac{L}{\delta}}}
    \\
    & = 3|X| \sqrt{2 T |A| \ln{\frac{T|X||A|}{\delta}}} + 2 |X|\sqrt{2T\ln{\frac{L}{\delta}}}
\end{align*}
where in the last step we used Jensen's inequality for the concave function $f(x,y)=\sqrt{xy}$ and the fact that $\sum_{s=0}^{k-1} |X_s| \leq |X|$.

Summing up for all $k=0,\dots,L-1$ finishes the proof.
\end{proof}